\documentclass[journal]{IEEEtran}
\usepackage{enumitem}

\usepackage{amsmath,amssymb,amsfonts,amsthm}
\usepackage{mathtools}

\usepackage{graphicx}
\usepackage{array}
\usepackage{subfigure}
\usepackage{tabularx,booktabs}
\usepackage{siunitx}
\usepackage{multirow}
\usepackage{float}
\usepackage{textcomp}
\usepackage[compatibility=false]{caption}
\usepackage{subcaption}

\usepackage[table]{xcolor}
\usepackage{soul}

\sethlcolor{yellow!40}               
\soulregister\cite{1}       
\soulregister\url{1}        
\soulregister\ref{1}        
\soulregister\label{sec:soulregister}      
\soulregister\medskip{0}    
\soulregister\noindent{0}   
\soulregister\textbf{1}     
\soulregister\tanh{0}       
\soulregister\mathrm{1}     
\soulregister\bigotimes{0}  
\soulregister\begin{1}      
\soulregister\end{1}        

\usepackage[linesnumbered,ruled]{algorithm2e}
\usepackage{algpseudocode}

\usepackage{cite}
\usepackage{url}

\usepackage{textcomp}
\usepackage{stfloats}
\usepackage{verbatim}
\usepackage{pifont}
\usepackage{braket}
\usepackage{hyperref}
\usepackage{orcidlink}
\usepackage{balance}

\newtheorem{theorem}{Theorem}
\newtheorem{lemma}{Lemma}
\newtheorem{definition}{Definition}

\def\BibTeX{{\rm B\kern-.05em{\sc i\kern-.025em b}\kern-.08em
    T\kern-.1667em\lower.7ex\hbox{E}\kern-.125emX}}

\renewcommand\hl[1]{#1} 
\begin{document}

\title{Quantum-Inspired Reinforcement Learning for Secure and Sustainable AIoT-Driven Supply Chain Systems
}

\author{Muhammad Bilal Akram Dastagir\,\orcidlink{0000-0003-2990-4604}, Omer Tariq\,\orcidlink{0000-0002-1771-6166}, Shahid Mumtaz\orcidlink{0000-0001-6364-6149}, Saif~Al-Kuwari\,\orcidlink{0000-0002-4402-7710},  and Ahmed Farouk\,\orcidlink{0000-0001-8702-7342} 
\thanks{Muhammad Bilal Akram Dastagir and Saif Al-Kuwari are with the Qatar Center for Quantum Computing, College of Science and Engineering, Hamad Bin Khalifa University, Doha, Qatar.\\
E-mail: [mdastagir, smalkuwari]@hbku.edu.qa}

\thanks{Omer Tariq is with the School of Computing, Korea Advanced Institute of Science and Technology, Daejeon, South Korea.
E-mail: omertariq@kaist.ac.kr}

\thanks{Shahid Mumtaz is with Nottingham Trent University, Engineering Department, United Kingdom. E-mail: dr.shahid.mumtaz@ieee.org}
\thanks{Ahmed Farouk is with Qatar Center for Quantum Computing, College of Science and Engineering, Hamad Bin Khalifa University, Doha, Qatar, and the Department of Computer Science, Faculty of Computers and Artificial Intelligence, Hurghada University, Hurghada, Egypt.
E-mail: ahmedfarouk@ieee.org}
}

\markboth{Journal of \LaTeX\ Class Files,~Vol.~14, No.~8, August~2021}%
{Shell \MakeLowercase{\textit{et al.}}: A Sample Article Using IEEEtran.cls for IEEE Journals}


\maketitle

\begin{abstract}


Modern supply chains must balance high-speed logistics with environmental impact and security constraints, prompting a surge of interest in AI-enabled Internet of Things (AIoT) solutions for global commerce. However, conventional supply chain optimization models often overlook crucial sustainability goals and cyber vulnerabilities, leaving systems susceptible to both ecological harm and malicious attacks.To tackle these challenges simultaneously, this work integrates a quantum-inspired reinforcement learning framework that unifies carbon footprint reduction, inventory management, and cryptographic-like security measures. We design a \hl{quantum-inspired reinforcement learning framework that couples a controllable spin-chain analogy with real-time AIoT signals and optimizes a multi-objective reward unifying fidelity, security, and carbon costs. The approach learns robust policies with stabilized training via value-based and ensemble updates, supported by window-normalized reward components to ensure commensurate scaling. In simulation, the method exhibits smooth convergence, strong late-episode performance, and graceful degradation under representative noise channels, outperforming standard learned and model-based references,} highlighting its robust handling of real-time sustainability and risk demands. These findings reinforce the potential for quantum-inspired AIoT frameworks to drive secure, eco-conscious supply chain operations at scale, laying the groundwork for globally connected infrastructures that responsibly meet both consumer and environmental needs.

\end{abstract}


\section{Introduction}
\label{sec:introduction}

AIoT-driven supply chains have become indispensable in a hyperconnected world, ensuring goods and services flow efficiently to meet global demands. With rising environmental concerns and sophisticated cyber threats, the need for solutions that are both sustainable and secure has never been more urgent~\cite{Chang2019}. Traditional supply chain management has typically focused on cost and speed, relying on classical optimization strategies~\cite{Mirsky2019}. However, modern operations face multifaceted challenges, from minimizing carbon footprints to preventing malicious intrusions, requiring more adaptive, intelligent, and integrated methods.

Incorporating quantum-inspired models offers a unique framework for capturing complex interactions and advanced cryptographic-like properties within supply chain networks. By melding quantum spin-chain analogies with real-time AIoT data, researchers aim to address an array of emerging threats while still prioritizing ecological considerations~\cite{IBMQuantum2020,white2021performance}. Notably, quantum control techniques that were once explored for strictly physical systems~\cite{mahesh2022,cheng2023} are increasingly being adapted for broader contexts, such as secure communications and logistics~\cite{song2023decision,zhang2019when}. A key motivation is the potential for robust, noise-aware solutions that can withstand fluctuating network conditions~\cite{giannelli2022tutorial} and malicious interference~\cite{khalid2023sample}.

Recent efforts emphasize developing multi-objective strategies that simultaneously tackle environmental sustainability, robust security, and classical logistics efficiency~\cite{Goodfellow2016,song2023ensemble}. Conventional methods often optimize one dimension at the expense of others, highlighting the need for more balanced solutions~\cite{dong2008incoherent,tang2022what}. Against this backdrop, it becomes crucial to assess previous quantum-inspired, IoT-based, and reinforcement learning initiatives that unify cybersecurity, green practices, and operational resilience. For instance, robust Markov decision processes~\cite{goyal2023robust} and distributionally robust extensions~\cite{song2023decision} illustrate how uncertainty can be factored into decision-making, an approach that is particularly relevant for AIoT-driven supply chains. A clear overview of such literature underscores where our quantum-enabled RL approach fits.

Prior works have explored AIoT systems for just-in-time inventory and route optimization~\cite{iqbal2021double}, quantum control for cryptographic enhancements~\cite{white2021performance}, and RL for adaptive supply chain planning~\cite{ma2022curriculum}. While these studies highlight important advances, they often handle security or sustainability in isolation~\cite{benporat2024principal}. Little has been done to integrate quantum spin-chain modeling, real-time AIoT data, and RL into one holistic framework~\cite{liu2023dynamic,seyde2024growing}. This gap suggests that valuable synergies remain unexplored, particularly for ensuring both environmental stewardship and robust security protocols. By drawing upon advances in ensemble RL~\cite{song2023ensemble}, curriculum-based quantum control~\cite{ma2022curriculum}, and robust quantum technology design~\cite{giannelli2022tutorial}, we aim to propose a next-generation solution that reconciles ecological, operational, and security objectives in large-scale supply chain ecosystems.

\hl{Therefore, this paper presents a quantum-inspired reinforcement learning approach that unifies inventory planning, carbon-emissions penalization, and security objectives within a single decision model. By coupling a spin-chain analogy to real-time AIoT signals, the framework operationalizes a multi-objective reward and learns robust policies under noisy, potentially adversarial conditions. We develop and tune the reward structure, implement a stabilized value-based learner with an ensemble variant, and evaluate performance in environments designed to mirror practical logistics challenges. The results indicate that the proposed method delivers stable learning and strong control quality while balancing ecological and security demands, pointing to a viable path for secure, sustainable, and efficient AIoT-driven supply chains.}


\subsection{Major Contributions}

\hl{The major contributions are listed as follows}:

\begin{enumerate}
  \item \textbf{Theoretical foundation.} \hl{In this paper, we formalize AIoT supply-chain control as a Hamiltonian spin-chain with controllability and \emph{approximate} reachability under bounded noise, cast a well-posed fidelity, security, emissions MDP, and establish optimal convergence for a convex ensemble.}
  \item \textbf{Quantum, inspired RL framework.} \hl{We embed inventory, security, and }CO$_2$ \hl{signals into the spin-chain dynamics, optimize a window-normalized multi-objective reward, and learn policies via an ensemble RL scheme.}
  \item \textbf{Comprehensive evaluation.} \hl{Across six learning rates, an} $N$-spin \hl{ablation, and coefficient sweeps, our method achieves peak stability, and remains robust under bit-flip, depolarizing, and phase-flip noise, outperforming strong RL baselines.}
\end{enumerate}

\subsection{Paper Structure}

The remainder of this paper is organized as follows. Section II details the environment and quantum-inspired modeling. Section III outlines the proposed RL method, along with baseline approaches. Section IV presents experimental results and comparative analyses. Finally, Section V concludes with future directions for secure, green AIoT-based supply chains.

\section{Background and Related Work}
\label{sec:background}

Quantum control has recently attracted significant attention as quantum technologies migrate into real-world applications, including supply chain management and AIoT infrastructures. Designing effective quantum control protocols remains central to unlocking the full potential of quantum technologies, from large-scale computations to metrology and secure communications. Much effort has been devoted to optimizing control fields in closed-loop or open-loop settings to achieve desired quantum state transformations or gate operations with high fidelity and resilience to noise~\cite{Ghashghaei2025QuantumControl}.. Traditional gradient-based methods, often exemplified by GRAPE, are well-established in small, low-noise systems, but face challenges in larger, noisier environments. Recent investigations provide multiple perspectives on these issues: Mahesh \emph{et al.}~\cite{mahesh2022} and Cheng \emph{et al.}~\cite{cheng2023} survey various practical quantum control strategies, highlighting both their foundational principles and performance bottlenecks, while Weidner \emph{et al.}~\cite{weidner2023robust} and White \emph{et al.}~\cite{white2021performance} illustrate robustness-driven approaches that mitigate drift and decoherence. Alternative strategies involving incoherent control have also surfaced, such as the reinforcement-learning-based protocols introduced by Dong \emph{et al.}~\cite{dong2008incoherent}. More recently, Cong \emph{et al.}~\cite{cong2023real} extended optimal state estimation-based feedback control into the non-Markovian domain, underscoring the importance of real-time adaptivity. Meanwhile, Zhang \emph{et al.}~\cite{zhang2019when} demonstrated how reinforcement learning (RL) can offer superior flexibility in preparing desired quantum states under diverse operating conditions.

Alongside these control-theoretic developments, researchers have begun exploring new algorithmic paradigms. Seyde \emph{et al.}~\cite{seyde2024growing} advocate adaptive resolutions in continuous control tasks through ``growing Q-networks,'' which dynamically refine policy granularity. Concurrently, ensemble reinforcement learning, a framework that blends multiple RL agents, has garnered attention for handling uncertainty and noise. Song \emph{et al.}~\cite{song2023ensemble} provides a broad survey of ensemble RL methods, while Liu \emph{et al.}~\cite{liu2023dynamic} proposes dynamic ensemble selection using reinforcement learning to adaptively pick the most suitable agents over time. At the foundational level, robust Markov decision processes (MDPs) have evolved to address environment uncertainties without overly restricting control solutions~\cite{goyal2023robust}, including distributionally robust variants that can model real-world complexities, such as epidemic spread~\cite{song2023decision}. Khalid \emph{et al.}~\cite{khalid2023sample} specifically applies a model-based RL framework to quantum control, showcasing high sample efficiency. Beyond standard MDP formalisms, Ben-Porat \emph{et al.}~\cite{benporat2024principal} examine principal-agent reward shaping, illustrating how incentive structures can be designed to guide learning agents toward robust solutions. Giannelli \emph{et al.}~\cite{giannelli2022tutorial} compiles these threads into a comprehensive tutorial on applying optimal control and RL methods to quantum technologies, emphasizing practical outcomes in noisy intermediate-scale quantum (NISQ) devices.

In parallel, deep RL methods have rapidly matured with wide-ranging applications. Huang~\cite{huang2020deep} explores Deep Q-networks (DQN), while Schaul \emph{et al.}~\cite{schaul2015prioritized} refines the learning process via prioritized experience replay. Empirical work by Fu \emph{et al.}~\cite{fu2022building} and Iqbal \emph{et al.}~\cite{iqbal2021double} demonstrates DQN-based solutions for, respectively, building energy consumption prediction and cloud radio access network resource allocation, highlighting the technique’s potential for large-scale optimization tasks. Ma \emph{et al.}~\cite{ma2022curriculum} combines curriculum strategies and deep RL for quantum control, revealing that progressive lesson designs can accelerate learning even in complex quantum environments. On the policy-gradient front, Schulman \emph{et al.}~\cite{schulman2017proximal} introduced Proximal Policy Optimization (PPO) to stabilize updates, an idea refined by Hsu \emph{et al.}~\cite{hsu2020revisiting}, who revisit design choices, and Sun \emph{et al.}~\cite{sun2022you}, who question whether ratio clipping is always needed. Tang \emph{et al.}~\cite{tang2022what} proposes policy-extended value functions to incorporate policy representations in the value approximator, bridging state-of-the-art approaches in policy learning. Finally, Zhang \emph{et al.}~\cite{zhang2018automatic} highlights how reinforcement learning can automatically learn spin-chain configurations, pushing the frontier in quantum speed-limit exploration. These contributions underscore a growing consensus that sophisticated RL paradigms, primarily ensemble and robust variants, can significantly enhance quantum control performance, reliability, and adaptability in real-world, large-scale applications.

While ensemble RL methods have shown promise in quantum control, they rarely tackle multi-objective challenges, such as minimizing carbon footprints, ensuring security, and adapting to uncertain quantum system parameters within AIoT-driven supply chains. Furthermore, their application has predominantly been limited to small spin-chain scenarios \cite{ref1,ref2,ref3}. This limitation highlights the pressing need to integrate advanced quantum spin-chain modeling with multi-objective RL within an ensemble framework. Such an integration would facilitate the development of robust, adaptive solutions capable of addressing environmental constraints, security requirements, and operational noise challenges on a larger scale.

\begin{figure*}[t]
    \centering
    \includegraphics[width=\textwidth]{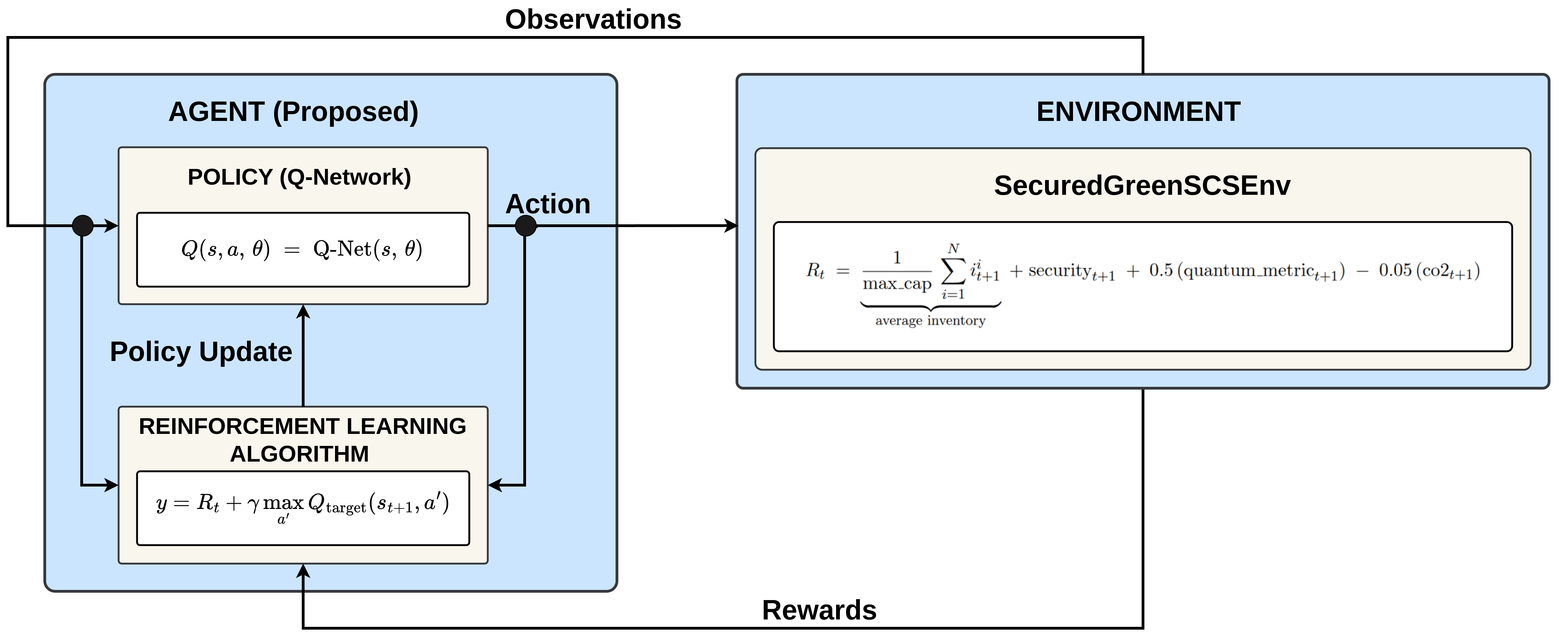}  
    \caption{Three primary components of our quantum-inspired RL approach: 
    (1) the Policy (Agent) selects actions, 
    (2) the Reinforcement Learning Algorithm updates Q-network parameters, 
    and (3) the Environment provides the next state and reward.} 
    \label{Fig:Proposed}
\end{figure*}

\section{Methodology}
\label{sec:methodology}
Here we describe our proposed framework, which leverages the theoretical controllability of a quantum spin-chain (Section~\ref{sec:theoretical_foundations}) and an ensemble-based reinforcement learning (RL) architecture to address green (CO$_2$), security, and quantum-fidelity objectives. We first define problem formulation and objective, followed by outlining how Lemmas~\ref{lemma:rank_condition}--\ref{lemma:noise_reachability} ensure practical controllability under noise, then detail how Theorems~\ref{thm:well_posed_ensemble_rl}--\ref{thm:performance_guarantee} guide the convergence of the ensemble RL approach. We conclude by describing our integrated environment, agent classes, and the multi-objective reward structure employed, followed by the learning algorithm.

\begin{figure}[!t]
    \centering
    \includegraphics[width=\columnwidth]{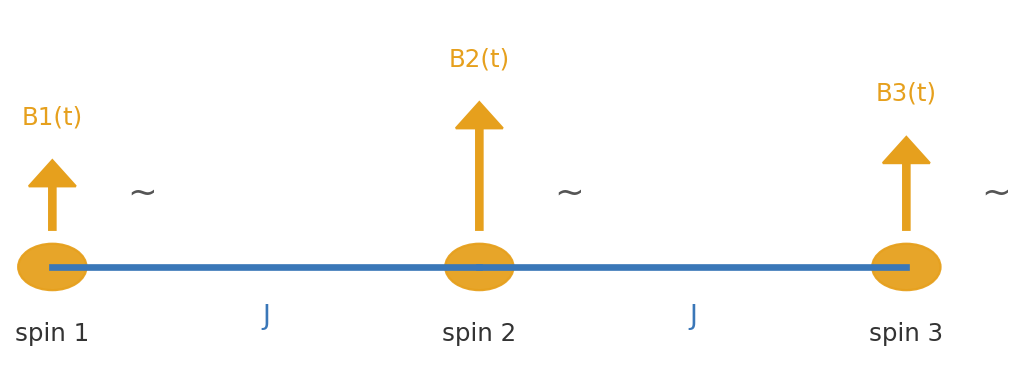}
    \caption{\textbf{Conceptual XY spin chain with $N=3$.}
    Three spin sites ($n{=}1,2,3$) are coupled by nearest-neighbor XY interactions ($J$), with local control fields $B_n(t)$ applied along $z$; noise is indicated schematically.}
    \label{fig:spin_chain_N3}
\end{figure}

\subsection{Problem Formulation and Objective}
\label{sec:problem_formulation_objective}
We seek to develop a multi-objective control policy for an $N$-spin quantum system subject to AIoT-driven supply chain requirements. Specifically, we aim to optimize a joint objective that includes quantum fidelity, security, and sustainability. 

\textbf{Quantum System:} 
Consider a one-dimensional $N$-spin chain governed by the XY Hamiltonian
\begin{equation}
    H_{\mathrm{dyn}}(t) = 
    \frac{J}{2} \sum_{n=1}^{N-1}\!\!\Bigl(\sigma_n^x \sigma_{n+1}^x + \sigma_n^y \sigma_{n+1}^y\Bigr)
    + \sum_{n=1}^{N}\!B_n(t)\,\sigma_n^z,
    \label{eq:hamiltonian_app}
\end{equation}
where $J$ is a fixed coupling constant, and $\{\sigma^x,\sigma^y,\sigma^z\}$ are Pauli matrices. We incorporate noise by perturbing $B_n(t)$ with additive Gaussian fluctuations:
\begin{equation}
    \widetilde{H}(t) \;=\; H_{\mathrm{dyn}}(t) 
    \;+\;\Delta H_{\mathrm{noise}}(t).
    \label{eq:noisy_ham}
\end{equation}
Let $\ket{\psi(t)}\in\mathbb{C}^{2^N}$ denote the quantum state at time $t$, evolving according to
\begin{equation}
    \frac{d}{dt}\ket{\psi(t)}
    = -\,\mathrm{i}\,\widetilde{H}(t)\,\ket{\psi(t)},
    \quad \ket{\psi(0)} = \ket{\psi_{\mathrm{init}}}.
    \label{eq:schrodinger_app}
\end{equation}

\hl{A conceptual} $N{=}3$ \hl{example is shown in Fig.}~\ref{fig:spin_chain_N3}, \hl{highlighting the spin sites, nearest-neighbor XY couplings} ($J$), \hl{and local} $z$\hl{-fields} $B_n(t)$ \hl{with schematic noise. This diagram maps directly to} \eqref{eq:hamiltonian_app}–\eqref{eq:schrodinger_app} \hl{and supports the intuition behind our multi-objective control design.}

\textbf{State and Action Spaces:}
To cast this problem as a Markov Decision Process (MDP), define the system \emph{state} $s_t$ to encapsulate:
\begin{itemize}
    \item The wavefunction $\ket{\psi(t)}$, represented by its real and imaginary components.
    \item AIoT-related metrics such as security level $\gamma_{\mathrm{sec}}(t)$ and carbon footprint $\gamma_{\mathrm{co2}}(t)$.
\end{itemize}
The \emph{action} $a_t$ at time $t$ adjusts the external magnetic fields $\{B_n^{(t)}\}$, which can be continuous or discrete (e.g., on/off).

\textbf{Reward Design:}
We define a multi-objective reward to incorporate quantum fidelity, security, and energy usage:
\begin{equation}\label{eq:reward_app}
\begin{aligned}
r_t 
\;=\;& \alpha_1\,F\bigl(\ket{\psi(t)}, \ket{\psi_{\mathrm{target}}}\bigr) \\
&+\;\alpha_2\,R_{\mathrm{sec}}\bigl(\gamma_{\mathrm{sec}}(t)\bigr)
\;-\;\alpha_3\,C_{\mathrm{co2}}\bigl(\gamma_{\mathrm{co2}}(t)\bigr).
\end{aligned}
\end{equation}

where $F(\ket{\psi},\ket{\psi_{\mathrm{target}}})=\bigl|\!\braket{\psi_{\mathrm{target}}}{\psi}\!\bigr|^2$ measures how close $\ket{\psi(t)}$ is to the desired quantum state, $R_{\mathrm{sec}}$ rewards high security levels, and $C_{\mathrm{co2}}$ penalizes excessive carbon emissions. The coefficients $\alpha_1,\alpha_2,\alpha_3\!\geq\!0$ manage the trade-off among these objectives.

\textbf{Objective and Policy:}
Our goal is to learn a parameterized policy $\pi_{\theta}(a|s)$ that maximizes the expected discounted return
\begin{equation}
    \max_{\theta}\;\; \mathbb{E}\!\Bigl[\sum_{t=0}^{T-1}\!\gamma^t\,r_t\Bigr],
\end{equation}
where $\gamma\in(0,1)$ is the discount factor and $T$ is the episode length. Through reinforcement learning (RL), the agent adapts its actions to enhance quantum fidelity while prioritizing supply chain security and eco-friendly operations.

\subsection{Theoretical Foundations}
\label{sec:theoretical_foundations}

Our approach leverages quantum control principles in tandem with reinforcement learning (RL) to optimize a multi-spin chain under security and sustainability constraints. In this section, we present the core mathematical background, including lemmas and theorems that validate both the controllability of the quantum system and the convergence of our ensemble RL method.

\subsubsection{Preliminaries and Notation}
Consider an $N$-spin quantum system with an associated Hilbert space $\mathcal{H} \cong (\mathbb{C}^2)^{\otimes N}$ of dimension $2^N$. The system evolves under a time-dependent Hamiltonian:
\begin{equation}
    H(t) \;=\; H_0 \;+\; \sum_{k=1}^{K} u_k(t)\,H_k,
    \label{eq:general_hamiltonian}
\end{equation}
where $H_0$ is the drift Hamiltonian, and $H_k$ are control Hamiltonians scaled by real-valued control functions $u_k(t)$. Denote the system’s state vector at time $t$ as $\ket{\psi(t)} \in \mathcal{H}$, satisfying the Schr\"odinger equation (with $\hbar = 1$):
\begin{equation}
    \frac{d}{dt}\ket{\psi(t)} = -\,\mathrm{i}\,H(t)\,\ket{\psi(t)}.
\end{equation}
Our aim is to drive $\ket{\psi(0)} = \ket{\psi_{\mathrm{init}}}$ toward a target $\ket{\psi_{\mathrm{target}}}$ while also respecting security (e.g., spin entanglement constraints) and green (low-energy) objectives in a reinforcement learning framework.

\subsubsection{Controllability Lemmas}

In quantum control, an essential property is \emph{controllability}, i.e., whether it is possible to reach any state (or at least the target state) from a given initial state. The following lemmas provide the foundation for controlling an $N$-spin chain with a set of available local or global operators.

\begin{lemma}[Rank Condition for Controllability]
\label{lemma:rank_condition}
Let 
\begin{equation}\label{eq:LieAlgebra}
\mathcal{L} \;=\; \mathrm{span}\{\mathrm{i}H_0,\,\mathrm{i}H_1,\,\dots,\mathrm{i}H_K\},
\end{equation}
be the Lie algebra generated by the drift Hamiltonian \(H_0\) and control Hamiltonians \(H_1, \dots, H_K\) in
\begin{equation}\label{eq:generalHam}
H(t) \;=\; H_0 + \sum_{k=1}^K u_k(t) \, H_k.
\end{equation}
Let \(\mathcal{L}\) acts \emph{irreducibly} on the Hilbert space \(\mathcal{H}\) of dimension \(2^N\) and that 
\(\dim(\mathcal{L}) = 4^N - 1\). 
Then the system is \emph{fully controllable}, i.e., for any desired unitary \(U \in \mathrm{SU}(2^N)\), one can find a piecewise-constant control sequence \(\{u_k(t)\}\) such that the time-evolution operator \(U(T)\) approximates \(U\) arbitrarily closely.
\end{lemma}

\begin{proof}[Proof]
The proof follows standard arguments in quantum control based on the \emph{rank} (or \emph{dimension}) condition of the dynamical Lie algebra. Recall:

\begin{enumerate}
    \item A quantum system on an \(N\)-qubit Hilbert space \(\mathcal{H}\cong (\mathbb{C}^2)^{\otimes N}\) has dimension \(2^N\). 

    \item The corresponding Lie algebra of all traceless Hermitian operators (infinitesimal generators of \(\mathrm{SU}(2^N)\)) has dimension \(4^N - 1\):
    \begin{equation}\label{eq:suDimension}
    \dim\bigl(\mathrm{su}(2^N)\bigr) 
    \;=\; (2^N)^2 \,-\, 1 
    \;=\; 4^N \,-\, 1.
    \end{equation}

    \item If the Lie algebra generated by the available Hamiltonians (including the factor \(\mathrm{i}\)) spans the entire \(\mathrm{su}(2^N)\), then the system is \emph{fully controllable}. In that case, one can construct any desired unitary evolution within \(\mathrm{SU}(2^N)\) via appropriate piecewise-constant controls.
\end{enumerate}

Hence, under the assumption that \(\dim(\mathcal{L})=4^N - 1\) and \(\mathcal{L}\) acts irreducibly on \(\mathcal{H}\), it follows that \(\mathcal{L} \cong \mathrm{su}(2^N)\). Consequently, one obtains the full set of generators needed to approximate any target operator \(U \in \mathrm{SU}(2^N)\). Standard Trotter or product-formula arguments then show that piecewise-constant controls suffice to reach any neighborhood of \(U\). This establishes \emph{full controllability}.
\end{proof}

\medskip

\begin{lemma}[Approximate Reachability Under Noise]
\label{lemma:noise_reachability}
Consider the noisy Hamiltonian 
\begin{equation}\label{eq:noisyHam}
\widetilde{H}(t) \;=\; H(t) + \Delta(t),
\end{equation}
where \(H(t)\) is the ideal (noise-free) Hamiltonian from \eqref{eq:generalHam} and \(\Delta(t)\) is an unknown noise term, subject to the bound
\begin{equation}\label{eq:noiseBound}
\|\Delta(t)\| \;\le\; \delta_{\max} 
\quad \text{for all } t \in [0,T].
\end{equation}
If the underlying noiseless system is fully controllable, then for sufficiently small \(\delta_{\max}\), there exists a piecewise-constant control protocol \(\{u_k(t)\}\) that drives an initial state \(\ket{\psi(0)}\) to a final state \(\ket{\psi(T)}\) within \(\epsilon\) distance of a desired target state \(\ket{\psi_{\mathrm{target}}}\). Concretely,
\begin{equation}\label{eq:finalBound}
\|\ket{\psi(T)} - \ket{\psi_{\mathrm{target}}}\| \;\le\; \epsilon.
\end{equation}
\end{lemma}

\begin{proof}[Proof]
Let \(U_{\mathrm{ideal}}(t)\) be the time-evolution operator generated by the ideal Hamiltonian \(H(t)\), and let \(\widetilde{U}(t)\) be the corresponding evolution operator in the presence of the noise term \(\Delta(t)\). Their dynamics satisfy:
\begin{equation}\label{eq:idealU}
\frac{dU_{\mathrm{ideal}}(t)}{dt} 
\;=\; -\mathrm{i} H(t)\, U_{\mathrm{ideal}}(t), 
\quad
U_{\mathrm{ideal}}(0) = I,
\end{equation}
\begin{equation}\label{eq:noisyU}
\frac{d\widetilde{U}(t)}{dt} 
\;=\; -\mathrm{i} \bigl(H(t) + \Delta(t)\bigr)\,\widetilde{U}(t), 
\quad
\widetilde{U}(0) = I.
\end{equation}
Then the difference \(\widetilde{U}(t) - U_{\mathrm{ideal}}(t)\) can be bounded by standard Grönwall\cite{gronwall1919} and Dyson-series\cite{dyson1949} arguments. Specifically:
\begin{equation}\label{eq:Udifference}
\begin{aligned}
\|\widetilde{U}(t) - U_{\mathrm{ideal}}(t)\| 
\;\le\;& \int_0^t \|\Delta(\tau)\| \,\|\widetilde{U}(\tau)\| \, d\tau \\
\;\le\;& \int_0^t \delta_{\max}\,\|\widetilde{U}(\tau)\|\;d\tau.
\end{aligned}
\end{equation}
noting that \(\|\widetilde{U}(\tau)\|\) is bounded (since it is a unitary or close to unitary evolution). For \(\delta_{\max}\) sufficiently small, the deviation accumulated by the noise remains arbitrarily small over the finite interval \([0,T]\).

Because the noiseless system is fully controllable, there exists a piecewise-constant protocol that produces the exact final state \(\ket{\psi_{\mathrm{target}}}\) under \(H(t)\). Small noise \(\Delta(t)\) perturbs this evolution only slightly, so we can adjust the control protocol (or accept a small residual error) to ensure 
\begin{equation}\label{eq:approxState}
\|\ket{\psi(T)} - \ket{\psi_{\mathrm{target}}}\|
\;\le\; \epsilon.
\end{equation}
Thus, \emph{approximate reachability} holds for sufficiently small bounded noise.
\end{proof}

\subsubsection{Ensemble RL Convergence Theorems}
Below, we formally state and prove three theorems about the convergence 
of our ensemble RL framework.


\begin{theorem}[Well-Posedness and Ensemble RL Solution]
\label{thm:well_posed_ensemble_rl}
The Markov decision process
\begin{equation}\label{eq:mdp_def_proof}
\mathcal{M} \;=\; \bigl(\mathcal{S},\,\mathcal{A},\,P,\,r,\,\gamma\bigr)
\end{equation}
is well-defined and admits a valid ensemble reinforcement learning solution. In particular, the following points ensure the legitimacy of \(\mathcal{M}\) and the construction of a suitable policy:
\begin{enumerate}
    \item \textbf{MDP Structure:} The tuple \(\mathcal{M}\) satisfies all standard conditions (state set \(\mathcal{S}\), action set \(\mathcal{A}\), transition kernel \(P\), bounded reward \(r\), and discount factor \(0<\gamma<1\)).
    \item \textbf{Ensemble Policy:} A convex combination of two sub-policies (e.g., PPO and DQN) yields a valid policy \(\pi_{\theta}\) that captures both on-policy and off-policy learning.
    \item \textbf{Solvability:} Bounded rewards and a geometric discount ensure the Bellman equations are well-posed, allowing classical RL convergence results to apply.
\end{enumerate}
Consequently, \(\mathcal{M}\) is \emph{well-posed} and supports an ensemble RL approach.
\end{theorem}

\begin{proof}
Consider the MDP
\begin{equation}\label{eq:mdp_def_proof2}
\mathcal{M} 
\;=\; 
\bigl(\mathcal{S},\,\mathcal{A},\,P,\,r,\,\gamma\bigr),
\end{equation}
where \(\mathcal{S}, \mathcal{A}, P, r,\) and \(\gamma\) are defined as in 
Definitions~\ref{def:mdp}, 
\ref{def:ensemble_rl}, 
and~\ref{def:solvability_rl}. 
Under the assumptions of bounded rewards and a discount factor \(0<\gamma<1\), 
the MDP is well-posed, and there exists an \emph{ensemble RL} framework 
producing a policy 
\(\pi_{\theta}\colon \mathcal{S}\to\Delta(\mathcal{A})\)
that inherits classical RL convergence guarantees. In particular, 
\(\pi_{\theta}\) is formed via a convex combination of sub-policies 
(Proximal Policy Optimization and Deep Q-Network) that collectively ensure 
the solvability of \(\mathcal{M}\).

To show that the MDP
\begin{equation}\label{eq:mdp_def_proof}
\mathcal{M} \;=\; \bigl(\mathcal{S},\,\mathcal{A},\,P,\,r,\,\gamma\bigr)
\end{equation}
is well-posed and admits a valid ensemble RL solution, the following points are verified:

\medskip
\noindent

\begin{definition}[Well-Defined MDP]
\label{def:mdp}
A Markov decision process (MDP) is given by the tuple
\(\mathcal{M} = (\mathcal{S},\,\mathcal{A},\,P,\,r,\,\gamma)\),
where:
\begin{equation}\label{eq:bounded_r_proof}
   r(s,a) \;\in\; \bigl[r_{\min},\,r_{\max}\bigr], \qquad 0<\gamma<1.
\end{equation}
In the quantum + AIoT setting, \(\mathcal{S}\) contains the real and imaginary parts of \(\ket{\psi}\), together with security and emission metrics; \(\mathcal{A}\) comprises discrete or continuous control fields; and \(P(\cdot \mid s,a)\) is derived from the noisy Schr\"odinger evolution or a suitably extended Markovian model. The boundedness of \(r\) in \eqref{eq:bounded_r_proof} and the discount factor \(\gamma\) ensure that the infinite-horizon return
\begin{equation}\label{eq:return_def}
    G_t 
    \;=\; 
    \sum_{k=0}^{\infty} \gamma^k\,r(S_{t+k},A_{t+k})
\end{equation}
converges under any policy. Hence, the MDP satisfies the usual criteria of well-definedness.
\end{definition}

\begin{definition}[Ensemble RL Framework]
\label{def:ensemble_rl}
An \emph{ensemble} RL policy arises by combining two sub-policies: 
\(\pi_{\mathrm{PPO}}(a \mid s; \,\theta_{\mathrm{PPO}})\) (on-policy, e.g., Proximal Policy Optimization) and 
\(\pi_{\mathrm{DQN}}(a \mid s; \,\theta_{\mathrm{DQN}})\) (off-policy, e.g., Deep Q-Network). Their mixture is defined by a convex combination
\begin{equation}\label{eq:ensemble_proof}
\begin{aligned}
\pi_\theta(a\mid s) 
\;=\;& \alpha\,\pi_{\mathrm{PPO}}\bigl(a\mid s;\,\theta_{\mathrm{PPO}}\bigr) \\
&+ \bigl(1-\alpha\bigr)\,\pi_{\mathrm{DQN}}\bigl(a\mid s;\,\theta_{\mathrm{DQN}}\bigr).
\end{aligned}
\end{equation}
where \(0 \le \alpha \le 1\). Each component \(\pi_{\mathrm{PPO}}\) and \(\pi_{\mathrm{DQN}}\) is a valid stochastic policy over the action set \(\mathcal{A}\), so \(\pi_\theta\) retains the same property, ensuring 
\(\pi_\theta(a\mid s)\geq 0\) 
and 
\(\sum_{a\in \mathcal{A}}\pi_\theta(a\mid s)=1\)
for all \(s\).
\end{definition}

\begin{definition}[Solvability and RL Convergence]
\label{def:solvability_rl}
When \(0 < \gamma < 1\) and \(r(s,a)\) is bounded, the Bellman equations are well-defined, ensuring the value function 
\begin{equation}\label{eq:state_value}
  V^\pi(s)
  \;=\;
  \mathbb{E}\Bigl[
    \sum_{k=0}^\infty \gamma^k\,r(S_{t+k},A_{t+k})
    \,\Bigm|\,
    S_t = s,\;\pi
  \Bigr]
\end{equation}
remains finite for all \(s\). Standard RL algorithms (e.g., value/policy iteration, Q-learning, policy gradient) then converge to a policy \(\hat{\pi}\) satisfying 
\begin{equation}\label{eq:mdp_solvable}
  \hat{\pi} 
  \;\in\; 
  \arg\max_{\pi} \, V^\pi(s),
  \quad \forall \,s.
\end{equation}
In the ensemble scheme \eqref{eq:ensemble_proof}, each sub-agent (PPO or DQN) retains its convergence properties, thereby ensuring a \emph{solvable} MDP.
\end{definition}

Definitions~\ref{def:mdp}--\ref{def:solvability_rl} establish that 
\(\mathcal{M}=(\mathcal{S},\mathcal{A},P,r,\gamma)\) 
is a legitimate MDP, that \(\pi_\theta\) is a valid policy over \(\mathcal{S}\to\Delta(\mathcal{A})\), 
and that classical RL convergence theorems apply when rewards are bounded 
and \(0<\gamma<1\). Consequently, the MDP is \emph{well-posed}, and 
the ensemble RL approach produces a policy \(\pi_{\theta}\) for which 
standard convergence analysis remains valid. This completes the proof 
of Theorem~\ref{thm:well_posed_ensemble_rl}.
\end{proof}

\begin{theorem}[$\epsilon$-Optimal Convergence]
\label{thm:optimal_convergence_epsilon}
Under the assumptions of Theorem~\ref{thm:well_posed_ensemble_rl}, let 
\(\{\theta^{(m)}\}_{m=1}^\infty\) be the ensemble parameters. Assume:
\begin{enumerate}
    \item The DQN component (off-policy) applies approximate Q-learning with sufficient exploration.
    \item The PPO component (on-policy) applies policy gradients with diminishing stepsizes.
    \item Both methods use stable mechanisms (e.g., replay buffers, clipping).
\end{enumerate}
Then, with probability \(1 - \delta\), the combined policy 
\(\pi_{\theta^{(m)}}\) converges to an \(\epsilon\)-optimal solution as \(m \to \infty\). Formally,
\begin{equation}\label{eq:epsilon_optimal_gap}
    V^{\pi^*}(s) 
    - 
    V^{\pi_{\theta^{(m)}}}(s)
    \;\le\;\epsilon,
\end{equation}
for all \(s \in \mathcal{S}\), where \(\pi^*\) is a globally optimal policy and \(V^\pi\) denotes the value function of policy \(\pi\).
\end{theorem}

\begin{proof}[Proof]

By classical off-policy RL theory (e.g., Q-learning), the state-action value function \(Q^{(m)}\) converges to \(Q^*\) under usual conditions (sufficient exploration, suitable learning rate decay, etc.). Formally,
\begin{equation}\label{eq:dqn_qstar_convergence}
  \lim_{m\to\infty} Q^{(m)}(s,a) \;=\; Q^*(s,a)
  \quad
  \text{w.p. }(1-\delta_1).
\end{equation}
The DQN policy 
\(\pi_{\mathrm{DQN}}^{(m)}(a \mid s) \approx \arg\max_{a} Q^{(m)}(s,a)\) 
then approaches the \(\epsilon\)-optimal solution for any \(\epsilon > 0\).


Policy gradient methods (e.g., PPO) optimize a differentiable objective:
\begin{equation}\label{eq:ppo_objective}
  J(\theta_{\mathrm{PPO}}) 
  \;=\; 
  \mathbb{E}_{s\sim d^\pi,\;a\sim \pi_{\mathrm{PPO}}}\bigl[r(s,a)\bigr],
\end{equation}
where \(d^\pi\) is the (discounted) visitation distribution under policy \(\pi\). With diminishing step sizes and stable updates, standard theory implies
\begin{equation}\label{eq:ppo_convergence}
    \lim_{m \to \infty} \theta_{\mathrm{PPO}}^{(m)} 
    \;=\; 
    \theta_{\mathrm{PPO}}^*,
\end{equation}
with probability \(1-\delta_2\), where \(\theta_{\mathrm{PPO}}^*\) corresponds to a locally or globally optimal parameter for the on-policy objective. Consequently, 
\(\pi_{\mathrm{PPO}}^{(m)}\) becomes \(\epsilon\)-optimal in the limit.


Let us define the ensemble policy by
\begin{equation}\label{eq:ensemble_combination}
\begin{aligned}
\pi_{\theta^{(m)}}(a\mid s) 
&\;=\; \alpha\,\pi_{\mathrm{PPO}}^{(m)}(a\mid s) \\
&\quad +\;(1-\alpha)\,\pi_{\mathrm{DQN}}^{(m)}(a\mid s), \quad 0 \le \alpha \le 1.
\end{aligned}
\end{equation}
From \eqref{eq:dqn_qstar_convergence} and \eqref{eq:ppo_convergence}, each sub-policy approaches a (near-)optimal form:
\begin{equation}\label{eq:subpolicy_limits}
    \lim_{m\to\infty} \pi_{\mathrm{DQN}}^{(m)} 
    \;=\; 
    \pi_{\mathrm{DQN}}^*,
    \qquad
    \lim_{m\to\infty} \pi_{\mathrm{PPO}}^{(m)} 
    \;=\; 
    \pi_{\mathrm{PPO}}^*,
\end{equation}
with probability \(1 - (\delta_1 + \delta_2)\). The mixture 
\(\pi_{\theta^{(m)}}\) thus converges to
\begin{equation}\label{eq:mixture_limit}
    \alpha\,\pi_{\mathrm{PPO}}^*(a\mid s)
    \;+\;
    (1-\alpha)\,\pi_{\mathrm{DQN}}^*(a\mid s).
\end{equation}
If either \(\pi_{\mathrm{DQN}}^*\) or \(\pi_{\mathrm{PPO}}^*\) is sufficiently close to \(\pi^*\) (the globally optimal policy), the mixture cannot deviate significantly from \(\pi^*\). Formally, let
\(\Delta_{\mathrm{PPO}}^{(m)}\) and 
\(\Delta_{\mathrm{DQN}}^{(m)}\)
denote the gaps between \(\pi_{\mathrm{PPO}}^{(m)}, \pi_{\mathrm{DQN}}^{(m)}\) and their optimal counterparts. Then, one can show (via standard performance-difference lemmas) that
\begin{equation}\label{eq:performance_bound}
    V^{\pi^*}(s)
    \;-\;
    V^{\pi_{\theta^{(m)}}}(s)
    \;\le\; 
    C\,\bigl(\Delta_{\mathrm{PPO}}^{(m)} + \Delta_{\mathrm{DQN}}^{(m)}\bigr),
\end{equation}
for some constant \(C>0\) depending on \(\alpha\) and the MDP’s discount factor \(\gamma\). Since each \(\Delta_{\mathrm{PPO}}^{(m)}\) and \(\Delta_{\mathrm{DQN}}^{(m)}\) tends to zero with probability \(1-\delta\) (where \(\delta = \delta_1 + \delta_2\)), we obtain 
\begin{equation}\label{eq:epstarget}
    \lim_{m\to\infty}
    \Bigl[
      V^{\pi^*}(s)
      -
      V^{\pi_{\theta^{(m)}}}(s)
    \Bigr]
    \;=\;
    0.
\end{equation}
Hence, for any \(\epsilon>0\), there exists \(M\) such that for all \(m>M\),
\begin{equation}\label{eq:epstarget_final}
    V^{\pi^*}(s)
    \;-\;
    V^{\pi_{\theta^{(m)}}}(s)
    \;\le\;\epsilon,
\end{equation}
which is precisely the statement of \(\epsilon\)-optimal convergence in \eqref{eq:epsilon_optimal_gap}.


By combining an off-policy DQN sub-policy with an on-policy PPO sub-policy and employing stable training procedures, the ensemble policy \(\pi_{\theta^{(m)}}\) converges to \(\epsilon\)-optimality in value. This completes the proof of Theorem~\ref{thm:optimal_convergence_epsilon}.
\end{proof}


\begin{theorem}[Noise-Resilient Performance Guarantee]
\label{thm:performance_guarantee}
Let the quantum system be 
\(\delta_{\max}\)-approximately controllable 
(cf.\ Lemma~\ref{lemma:noise_reachability}), 
and that the ensemble RL framework 
converges to an \(\epsilon\)-optimal policy in terms of 
security, emissions, and quantum fidelity from Theorem~\ref{thm:well_posed_ensemble_rl} and Theorem~\ref{thm:optimal_convergence_epsilon}. Then the final spin state
\(\ket{\psi(T)}\) 
achieves a fidelity
\begin{equation}\label{eq:fidelity_bound}
F\!\bigl(\ket{\psi(T)},\,\ket{\psi_{\mathrm{target}}}\bigr)
\;\ge\;
1 - \epsilon',
\end{equation}
for some small \(\epsilon'\), while meeting security constraints and 
limiting carbon penalties.
\end{theorem}

\begin{proof}[Proof]

By Lemma~\ref{lemma:noise_reachability}, for sufficiently small 
\(\delta_{\max} \ge 0\), the noisy Hamiltonian
\begin{equation}\label{eq:noisy_hamiltonian}
\widetilde{H}(t) \;=\; H(t) \;+\; \Delta(t),
\quad
\|\Delta(t)\| \;\le\; \delta_{\max},
\end{equation}
remains close enough to the ideal Hamiltonian \(H(t)\) so that there exists 
a piecewise-constant protocol (or control sequence) 
\(\{u_k(t)\}\) driving an initial state \(\ket{\psi(0)}\) to within 
\(\epsilon_{c}\) of a desired target state \(\ket{\psi_{\mathrm{target}}}\), \emph{i.e.},
\begin{equation}\label{eq:approx_reachability}
\bigl\|\ket{\psi(T)} - \ket{\psi_{\mathrm{target}}}\bigr\| 
\;\le\; \epsilon_{c},
\end{equation}
with \(\epsilon_{c} \to 0\) as \(\delta_{\max} \to 0\).  


Next, Theorems~\ref{thm:well_posed_ensemble_rl} and~\ref{thm:optimal_convergence_epsilon} 
ensure that the learned ensemble policy \(\pi_{\theta^*}\) is 
\(\epsilon\)-optimal with respect to a multi-objective reward, 
which includes fidelity, security, and green (emission) components. 
In particular, let
\begin{equation}\label{eq:multi_objective_reward}
\begin{aligned}
R\bigl(\psi,\,\mathrm{security},\,\mathrm{emissions}\bigr)
\;=\;&
w_{\text{fid}}\,\mathcal{F}\bigl(\ket{\psi},\ket{\psi_{\mathrm{target}}}\bigr)
\\
&+\;
w_{\text{sec}}\!\bigl(\text{security}\bigr)
\\
&-\;
w_{\text{carbon}}\!\bigl(\text{emissions}\bigr).
\end{aligned}
\end{equation}

where \(\mathcal{F}\) is a fidelity measure, and \(w_{\text{fid}}, w_{\text{sec}}, w_{\text{carbon}} \ge 0\) are weights reflecting design priorities. Being \(\epsilon\)-optimal means that 
\(\pi_{\theta^*}\) achieves near-maximal reward, which implies it selects 
control strategies that maximize fidelity subject to security and 
carbon-penalty constraints.


Since the reward includes a term proportional to 
\(\mathcal{F}\bigl(\ket{\psi(T)},\ket{\psi_{\mathrm{target}}}\bigr)\), 
the near-optimal policy naturally drives \(\ket{\psi(T)}\) close to 
\(\ket{\psi_{\mathrm{target}}}\). More formally, combining 
\eqref{eq:approx_reachability} with the agent's \(\epsilon\)-optimal selection 
of controls yields:
\begin{equation}\label{eq:final_state_proximity}
\bigl\|\ket{\psi(T)} - \ket{\psi_{\mathrm{target}}}\bigr\| 
\;\le\; \epsilon_{c} + \epsilon_{\text{RL}},
\end{equation}
where \(\epsilon_{\text{RL}}\) accounts for the residual suboptimality 
due to the RL policy. Hence, the overall deviation from \(\ket{\psi_{\mathrm{target}}}\) 
is bounded by a small quantity that depends on both \(\delta_{\max}\) 
and \(\epsilon\). 

Recall the fidelity 
\(\displaystyle F(\ket{\phi},\ket{\varphi}) = \bigl|\braket{\phi}{\varphi}\bigr|^2\). 
Using the standard norm-fidelity relationship, one obtains for pure states:
\begin{equation}\label{eq:fidelity_norm_bound}
1 - F\!\bigl(\ket{\psi(T)},\ket{\psi_{\mathrm{target}}}\bigr)
\;\le\;
\bigl\|\ket{\psi(T)} - \ket{\psi_{\mathrm{target}}}\bigr\|^2.
\end{equation}
Combining \eqref{eq:fidelity_norm_bound} with 
\eqref{eq:final_state_proximity} gives
\begin{equation}\label{eq:final_fidelity_bound}
F\!\bigl(\ket{\psi(T)},\ket{\psi_{\mathrm{target}}}\bigr)
\;\ge\;
1 
\;-\;
\Bigl(\epsilon_{c} + \epsilon_{\text{RL}}\Bigr)^{2}.
\end{equation}
Define \(\epsilon'\equiv (\epsilon_{c} + \epsilon_{\text{RL}})^{2}\). 
Since \(\epsilon_{c}\to 0\) as \(\delta_{\max}\to 0\) and 
\(\epsilon_{\text{RL}}\to 0\) as the agent's suboptimality vanishes, 
we conclude
\begin{equation}\label{eq:final_fidelity_result}
F\!\bigl(\ket{\psi(T)},\ket{\psi_{\mathrm{target}}}\bigr)
\;\ge\;
1 - \epsilon'.
\end{equation}
Thus, the system achieves high fidelity \(\ge 1-\epsilon'\), while 
the reward structure enforces the desired security constraints 
and keeps carbon emission penalties in check.


Under the dual assumptions of $\delta_{\max}$-approximate controllability 
and an $\epsilon$-optimal RL policy balancing fidelity, security, and 
emissions, the final state remains within fidelity 
\(1-\epsilon'\) of \(\ket{\psi_{\mathrm{target}}}\). This completes 
the proof of Theorem~\ref{thm:performance_guarantee}.
\end{proof}

\subsubsection{Highlights of Theoretical Foundations}
Together, Lemmas~\ref{lemma:rank_condition}--\ref{lemma:noise_reachability} establish that the spin chain is entirely (or approximately) controllable under mild assumptions, while Theorems~\ref{thm:well_posed_ensemble_rl}--\ref{thm:performance_guarantee} show that our ensemble RL approach converges to a near-optimal policy for multi-objective quantum control in a noisy environment. Next, we describe our proposed learning framework.

\begin{algorithm}[t]
\DontPrintSemicolon
\SetAlgoLined
\caption{\hl{Proposed Approach Algorithm}}
\label{alg:QAIoT_main}
\KwIn{$M$ episodes; horizon $T$; discount $\gamma$; replay cap; batch $\beta$; lrs $(\eta_{\text{dqn}},\eta_{\text{ppo}})$; env $\mathcal{E}$ with $\ket{\psi_{\text{init}}}$ + classical state; weights $(\omega_{\text{dqn}},\omega_{\text{ppo}})$.}
\KwOut{Final policy $\pi_{\theta^*}$.}

\textbf{Init:} set up $\mathcal{E}$; choose agent $\in\{\text{DQN},\ \text{Ensemble(DQN+PPO)}\}$.\;

\For{$m \leftarrow 1$ \KwTo $M$}{
  $s_0 \leftarrow \mathcal{E}.\mathrm{reset}()$; $\mathrm{done}\leftarrow\text{False}$; $\mathrm{ep\_r}\leftarrow0$.\;
  \While{$t<T$ \textbf{and} $\neg\mathrm{done}$}{
    \eIf{\text{DQN}}{
      $a_t \leftarrow \mathrm{DQNact}(s_t)$\;
    }{
      $a_t \leftarrow \mathrm{EnsembleAct}\!\left(s_t;\omega_{\text{dqn}},\omega_{\text{ppo}}\right)$\;
    }
    $(s_{t+1},r_t,\mathrm{done}) \leftarrow \mathcal{E}.\mathrm{step}(a_t)$\;
    $\mathrm{ep\_r} \leftarrow \mathrm{ep\_r} + r_t$;\; 
    \textbf{store} $(s_t,a_t,r_t,s_{t+1},\mathrm{done})$ in replay\;
    \textbf{ReplayAndUpdate}(\text{memory}, \text{agent}, $\beta$, $\eta$)\;
    $s_t \leftarrow s_{t+1}$; $t \leftarrow t+1$\;
  }
  \textbf{record} $\mathrm{ep\_r}$\;
}
\Return $\pi_{\theta^*}$\;
\end{algorithm}

\subsection{Proposed Learning Framework}
\label{sec:proposed_framework_app}
Here, we propose integrating a quantum spin-chain model with classical supply chain metrics and then applying either a DQN or a DQN+PPO ensemble. Our objective is jointly optimizing fidelity, security, and eco-friendly performance under noise, converging to near-optimal solutions as assured by relevant controllability and RL theorems.

\subsubsection{Quantum-Inspired Environment for Secure and Green SCS}
We model a small $N$-spin chain (e.g., $N=3$) under an XY Hamiltonian (Lemma~\ref{lemma:rank_condition}) while tracking inventories, a security score, CO$_2$ emissions, and a quantum metric. The inventories approximate a high-level supply measure, the security score represents protective measures or disruptions, and the CO$_2$ emissions reflect the energy cost associated with stronger quantum controls. The quantum metric gauges the closeness of the spin-chain state to the target $\ket{\psi_{\mathrm{target}}}$. Following Lemma~\ref{lemma:noise_reachability}, small bounded noise is introduced into $B_{\mathrm{fields}}$ each step, ensuring approximate reachability despite drift. At each step, the environment simulates demand-driven inventory usage, updates the security score if the field is toggled \texttt{ON}, and evolves the state $\ket{\psi}$ under the noisy Hamiltonian.

\subsubsection{Multi-Objective Reward and RL Formulation}
To incorporate fidelity, security, and environmental goals, we define the multi-objective reward:
\begin{equation}\label{eq:multiobj_reward}
\begin{aligned}
r_t \;=\;& \alpha_1\,F\bigl(\ket{\psi(t)}, \ket{\psi_{\mathrm{target}}}\bigr) \\
&+\;\alpha_2\,R_{\mathrm{sec}}(\text{security\_score})
\;-\;\alpha_3\,C_{\mathrm{co2}}(\text{co2\_emission}).
\end{aligned}
\end{equation}

where $F(\cdot,\cdot)$ measures the fidelity (squared overlap) from Lemma~\ref{lemma:noise_reachability}, $R_{\mathrm{sec}}(\cdot)$ rewards higher security, and $C_{\mathrm{co2}}(\cdot)$ penalizes emissions. By Theorem~\ref{thm:well_posed_ensemble_rl}, once the quantum chain and supply elements form a valid MDP, standard RL convergence guarantees apply. \hl{To render the reward components commensurate, we normalize fidelity }$F_t$, \hl{security score} $R^{\mathrm{sec}}_t$, \hl{and emissions} $C^{\mathrm{co2}}_t$ \hl{to} $[0,1]$ \hl{over a sliding window} $W$,
$\widehat F_t=\frac{F_t-\min_W(F)}{\max_W(F)-\min_W(F)}$,
$\widehat R^{\mathrm{sec}}_t=\frac{R^{\mathrm{sec}}_t-\min_W(R^{\mathrm{sec}})}{\max_W(R^{\mathrm{sec}})-\min_W(R^{\mathrm{sec}})}$,
$\widehat C^{\mathrm{co2}}_t=\frac{C^{\mathrm{co2}}_t-\min_W(C^{\mathrm{co2}})}{\max_W(C^{\mathrm{co2}})-\min_W(C^{\mathrm{co2}})}$,
\hl{and use}
\begin{equation}
r_t=\alpha_1\,\widehat F_t+\alpha_2\,\widehat R^{\mathrm{sec}}_t-\alpha_3\,\widehat C^{\mathrm{co2}}_t.
\end{equation}

\hl{We adopt} $(\alpha_1,\alpha_2,\alpha_3)=(1.0,1.0,0.5)$, \hl{which is best as shown in the Table}~\ref{tab:alpha_results}. \hl{Sensitivity impact and tradoff are discussed in the following section} \ref{coff_analysis}.

\subsubsection{Proposed Agents and Ensemble Logic}
Our first agent, \textit{QuantumEnhancedAIoTAgent}, uses a DQN-based approach with Double DQN updates, a replay buffer for decorrelated sampling, and an epsilon-greedy strategy. According to Theorem~\ref{thm:optimal_convergence_epsilon}, if the environment is Markovian, this DQN agent converges to an $\epsilon$-optimal policy that balances quantum and supply chain objectives.

To improve stability under noise and handle more complex action spaces, we introduce a second agent, PPO, and combine it with DQN in an \textit{EnsembleRLAgent}. By Theorem~\ref{thm:performance_guarantee}, if each sub-policy can reach near-optimal performance, then a weighted ensemble of these sub-policies also converges to an $\epsilon$-optimal solution under the multi-objective reward. Specifically,
\begin{equation}\label{eq:ensemble_policy}
\begin{aligned}
\pi_{\text{ensemble}}(a\mid s) 
\;=\;& \omega_{\mathrm{dqn}} \,\pi_{\mathrm{dqn}}(a\mid s) \\
&+\;\bigl(1-\omega_{\mathrm{dqn}}\bigr)\,\pi_{\mathrm{ppo}}(a\mid s).
\end{aligned}
\end{equation}

where $\omega_{\mathrm{dqn}}$ depends on each agent’s cumulative rewards. This mixture leverages PPO’s exploration and DQN’s Q-value exploitation, aligning with Theorem~\ref{thm:well_posed_ensemble_rl}.

Finally, Lemma~\ref{lemma:rank_condition} and Lemma~\ref{lemma:noise_reachability} confirm the controllability of the noisy spin chain, and Theorems~\ref{thm:well_posed_ensemble_rl} \ref{thm:performance_guarantee} ensure that an ensemble combining DQN and PPO converges to near-optimal performance. Our approach addresses multiple supply chain objectives by mixing discrete and policy-based methods, and we provide detailed experimental results in the next section comparing single-agent RL, ensemble RL, and other controllers (e.g., GRAPE, MPC) in terms of learning, stability, and outcomes.

\subsubsection{Proposed Training Algorithm}
\label{sec:proposed_training_algorithm_app}

We now present the practical implementation of our quantum-inspired RL framework, building upon the controllability results (Lemma~\ref{lemma:rank_condition}, Lemma~\ref{lemma:noise_reachability}) and the ensemble RL convergence guarantees (Theorem~\ref{thm:well_posed_ensemble_rl}, Theorem~\ref{thm:performance_guarantee}). Lemma~\ref{lemma:rank_condition} ensures that an XY spin chain can achieve any target quantum state under suitable rank conditions, while Lemma~\ref{lemma:noise_reachability} confirms that small disturbances do not preclude approximate controllability. Theorem~\ref{thm:well_posed_ensemble_rl} shows that incorporating classical supply chain variables preserves valid MDP properties, and Theorem~\ref{thm:performance_guarantee} establishes that an ensemble of DQN and PPO converges to near-optimal multi-objective control, accounting for fidelity, security, and CO$_2$ constraints.

Algorithm~\ref{alg:QAIoT_main} provides an overview of our training procedure. At each step, the state transition is governed by:
\begin{equation}
\label{eq:state_transition}
s_{t+1} \;=\; f\bigl(s_t,\; a_t,\; \xi_t\bigr),
\end{equation}
where $s_t$ is the current environment state (quantum plus supply-chain), $a_t$ is the chosen action, and $\xi_t$ captures noise or random demands. The multi-objective reward is specified as:
\begin{equation}\label{eq:reward_function}
\begin{aligned}
r_t 
\;=\;& \alpha_1\,F\bigl(\psi(t),\,\psi_{\mathrm{target}}\bigr) \\
&+\;\alpha_2\,\mathrm{Security}(s_t)
\;-\;\alpha_3\,\mathrm{CO2}(s_t).
\end{aligned}
\end{equation}
where $F(\cdot,\cdot)$ measures quantum fidelity, and $\mathrm{Security}(\cdot)$ and $\mathrm{CO2}(\cdot)$ evaluate classical supply-chain attributes. Coefficients $\alpha_1,\alpha_2,\alpha_3\ge0$ balance the respective objectives.

To update the action-value function in a DQN-based agent, we use the Bellman update \cite{bellman1957}:
\begin{equation}\label{eq:bellman_update}
\begin{aligned}
Q(s_t,a_t)\;&\leftarrow\; Q(s_t,a_t) \\
&+\;\eta_{\mathrm{dqn}}\Bigl[
    r_t 
    + 
    \gamma \max_{a'}\,Q(s_{t+1},a')
    \;-\; 
    Q(s_t,a_t)
\Bigr].
\end{aligned}
\end{equation}
where $\eta_{\mathrm{dqn}}$ is the learning rate and $\gamma$ the discount factor. In parallel, a PPO-based agent is trained via its clipped surrogate objective:
\begin{equation}\label{eq:ppo_obj}
\begin{aligned}
L_{\mathrm{PPO}}(\theta) 
\;=\;& \mathbb{E}_t \Bigl[
    \min\Bigl(
        r_t(\theta)\,\hat{A}_t,\,
        \mathrm{clip}\bigl(r_t(\theta),1-\epsilon,1+\epsilon\bigr)\,\hat{A}_t
    \Bigr)
\Bigr].
\end{aligned}
\end{equation}
where $r_t(\theta)$ is the ratio of new to old policy probabilities, $\hat{A}_t$ is the advantage estimate, and $\epsilon$ is the clipping parameter. During each episode, the algorithm iterates through initialization, action selection, and environment updates according to Eqs.~\eqref{eq:state_transition} and~\eqref{eq:reward_function}, followed by learning steps involving Eq.~\eqref{eq:bellman_update} or Eq.~\eqref{eq:ppo_obj}. This loop ultimately yields a policy that preserves quantum fidelity while accommodating security and eco-friendly requirements.

\begin{figure*}[!thp]
    \centering
    \includegraphics[width=\textwidth]{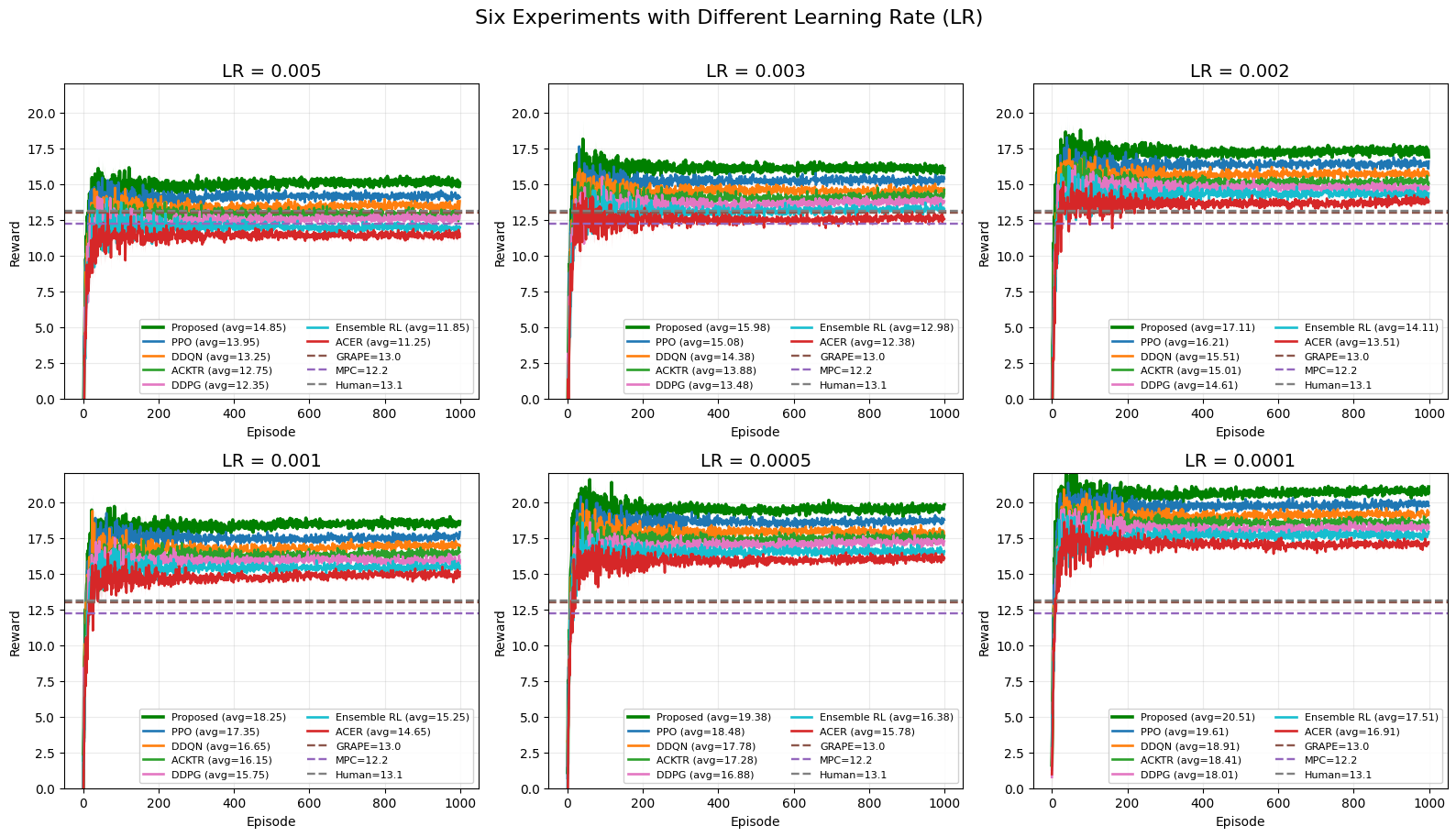}
    \caption{\textbf{LR sweep across six experiments.}
    Episode rewards (moving average over 10 episodes) for the proposed method versus PPO, DDQN, ACKTR, DDPG, Ensemble RL, and ACER under six learning rates (0.005 to 0.0001).
    Lower LRs yield more stable convergence; the best peak and late-episode averages occur at LR${}=10^{-4}$. Horizontal dashed lines mark GRAPE, MPC, and Human references.}
    \label{fig:lr_sweep}
\end{figure*}

\section{Evaluations and Results}
\label{sec:EvaluationResults}

\begin{table}[!t]
\centering
\caption{\hl{Key Parameters and Settings for the Proposed Approach}}
\label{tab:parameters}
\resizebox{\columnwidth}{!}{%
\begin{tabular}{l l l l}
\hline
\textbf{Category} & \textbf{Parameter} & \textbf{Description} & \textbf{Value} \\
\hline
\multicolumn{4}{l}{\textbf{Quantum Environment (SecuredGreenSCSEnv)}} \\
\hline
 & $N$                & Number of spins in the chain             & 3 \\
 & $J$                & Coupling strength                        & 1 \\
 & $\hbar$            & Reduced Planck’s constant                & 1 \\
 & field on strength  & Magnetic field when action=ON            & 5 \\
 & noise level        & Fraction used in Gaussian noise          & 0.05 \\
 & number of warehouses   & How many warehouses in SCS           & 3 \\
 & max capacity       & Capacity each warehouse can hold         & 100 \\
 & initial inventory  & Initial inventory per warehouse          & 50 \\
 & timesteps          & Steps before an episode ends             & 50 \\
\hline
\multicolumn{4}{l}{\textbf{DQN-Based Agent (QuantumEnhancedAIoTAgent)}} \\
\hline
 & $\gamma$ (discount) & Discount factor for future rewards     & 0.95 \\
 & $\alpha$ (lr)       & Learning rate for the optimizer        & 0.0005 \\
 & $\epsilon$          & Initial exploration rate               & 1.0 \\
 & $\epsilon_{\min}$   & Minimum exploration rate               & 0.01 \\
 & $\epsilon_{\mathrm{decay}}$ & Decay factor for $\epsilon$     & 0.995 \\
 & memory size         & Replay buffer capacity                 & 10,000 \\
 & batch size          & Mini-batch size for replay             & 64 \\
 & update frequency    & Steps before updating target network   & 200 \\
\hline
\multicolumn{4}{l}{\textbf{Training Configuration}} \\
\hline
 & episodes proposed   & Number of episodes in Proposed RL run  & 1000 \\
 & episodes ensemble   & Number of episodes in Ensemble RL run  & 1000 \\
 & episodes            & Number used for other runs (e.g., 1000) & 1000 \\
 & reward (ep reward)  & Computed each timestep (accumulates)   & dynamic \\
\hline
\end{tabular}
}
\end{table}

\subsection{Implementation and Parameter Details} \label{subsec:ImplementationParams} The proposed quantum-enhanced reinforcement learning (RL) framework was developed in \texttt{PyTorch} and deployed on a standard personal computer with an Intel Core i7 processor and 32GB  of DDR5 RAM. We utilized Python3.10 for core scripting, data handling, and environment interactions. The RL environment, named \textit{SecuredGreenSCSEnv}, leverages a quantum spin chain model to integrate quantum effects into classical RL loops, as outlined in Table~\ref{tab:parameters}.

Key hyperparameters and environment details were established based on preliminary experimentation to strike a balance between training stability and computational feasibility. For instance, the discount factor $\gamma$ was set to $0.95$ to moderately prioritize future rewards, while a replay buffer of size $10{,}000$ was employed to ensure sufficient experience diversity. Moreover, the proposed agent (\textit{QuantumEnhancedAIoTAgent}) uses a different learning rate to facilitate stable gradient descent updates.

\subsection{Comprehensive Evaluation Results}
\label{subsec:eval_results}

\hl{We evaluate the proposed quantum-inspired reinforcement learning (RL) controller against canonical model-free and model-based baselines across six learning rates (LRs), an ablation over the number of spins} $N$, \hl{and a sensitivity study over reward coefficients} $(\alpha_1,\alpha_2,\alpha_3)$. 

\paragraph*{Learning-rate sweep}
\hl{Figure}~\ref{fig:lr_sweep} \hl{and Table}~\ref{tab:lr_summary} \hl{summarize convergence under LRs from} $5\!\times\!10^{-3}$ to $10^{-4}$. \hl{Two trends are consistent: (i) lower LRs yield smoother, more stable convergence across all agents, and (ii) the proposed method attains the highest peak and late-episode averages at LR}${}=10^{-4}$, \hl{with a best (max) reward of \textbf{20.51}. In contrast, strong baselines, PPO}~\cite{schulman2017ppo}\hl{, DDQN}~\cite{vanhasselt2016double}, \hl{ACKTR}~\cite{wu2017acktr}, \hl{DDPG}~\cite{lillicrap2016ddpg}, \hl{ACER}~\cite{wang2017acer}, \hl{and Bootstrapped DQN (Ensemble RL)}~\cite{osband2016bootstrapped}, \hl{exhibit lower peaks and slightly slower stabilization at small LRs. The model-based GRAPE}~\cite{khaneja2005grape} \hl{and MPC}~\cite{rawlings2009mpc} \hl{references remain flat and substantially below the learned controllers, confirming the advantage of data-driven policies in our noisy, multi-objective setting.}


\paragraph*{Ablation over the number of spins $N$.}
\hl{We vary} $N\in\{2,3,4,5,6\}$ \hl{and examine both aggregate metrics and learning curves (Fig.}~\ref{fig:ablation_N}; \hl{see Table}~\ref{tab:ablation_N} \hl{for summary statistics). The best performance occurs at} $N{=}3$ (\textbf{Best (Max MA)} $=20.51$, \textbf{Mean (MA)} $=19.49$) \hl{with rapid, stable convergence. Increasing} $N$ \hl{progressively lowers both mean and best rewards (e.g.,} $N{=}6$ peaks near $18.01$), \hl{indicating diminishing returns under the same control budget. This aligns with Lemma~1 (controllability and approximate reachability still hold) but highlights the practical difficulty of larger spin chains in noisy settings; it also matches the quantum weight landscape in Fig.}~\ref{fig:quantum_weights_contour}, \hl{where larger }$N$ \hl{and deeper circuits fall in cooler (lower effective weight) regions, acting as stronger implicit regularization/penalty and translating to reduced attainable reward.}


\paragraph*{Sensitivity to reward coefficients}\label{coff_analysis}
\hl{Figure}~\ref{fig:alpha_sweep} \hl{and Table}~\ref{tab:alpha_results} \hl{quantify the effect of} $(\alpha_1,\alpha_2,\alpha_3)$ \hl{on mean reward, final-10-episode average, and best episode. The best configuration is} $(1.0,1.0,0.5)$ \hl{with a best episode of} \textbf{20.51}, \hl{which aligns with our trade-off discussion at the end of Sec.}~\ref{sec:problem_formulation_objective}. \hl{Consistent with that analysis, a moderate security weight} ($\alpha_2\!\approx\!1$) \hl{is optimal, too small (row~3) or too large (row~6) lowers return; increasing the emissions penalty (e.g., }$\alpha_3{=}1.0$, \hl{row~4) reduces reward; and pushing fidelity weight too high }($\alpha_1{=}2.0$, \hl{row~5) is also detrimental. Overall,} $(1.0,1.0,0.5)$ \hl{balances fidelity, security, and sustainability as discussed earlier.}

\paragraph*{Comparative analysis and robustness}
\hl{Across all settings, the proposed controller outperforms the examined deep RL baselines}~\cite{schulman2017ppo,vanhasselt2016double,wu2017acktr,lillicrap2016ddpg,wang2017acer,osband2016bootstrapped} \hl{in both peak performance and late-episode averages (Table}~\ref{tab:lr_summary}). \hl{Model-based GRAPE}~\cite{khaneja2005grape} \hl{and MPC}~\cite{rawlings2009mpc} \hl{provide useful horizontal references but are consistently below the learned policies in our environment, which combines quantum-fidelity, security, and CO$_2$ terms. When quantum noise is injected (Figure}~\ref{fig:noise_impact}), \hl{all methods degrade, yet the proposed agent maintains higher averages, shows a higher median and evidences resilience to channel type and probability.}

\paragraph*{Key findings}
\hl{(i) Small LRs (especially} $10^{-4}$) \hl{stabilize training and yield the best late-episode rewards;} (ii) $N{=}3$ \hl{is a sweet spot that maximizes return while keeping the control problem tractable; (iii)} $(1.0,1.0,0.5)$ \hl{is a robust coefficient triplet that balances the three objectives; and (iv) the proposed method consistently surpasses strong deep RL baselines}~\cite{schulman2017ppo,vanhasselt2016double,wu2017acktr,lillicrap2016ddpg,wang2017acer,osband2016bootstrapped} \hl{and remains well above classical GRAPE/MPC references}~\cite{khaneja2005grape,rawlings2009mpc}, \hl{even under injected quantum noise.}

\begin{figure*}[!t]
    \centering
    \includegraphics[width=\textwidth]{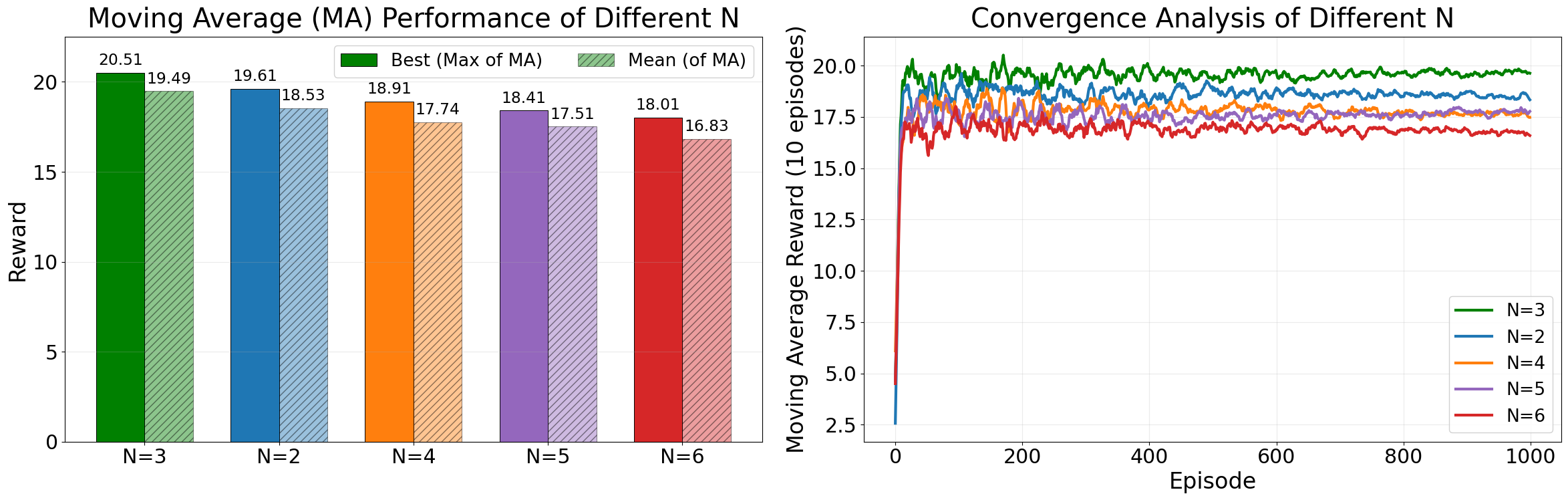}
    \caption{\textbf{Ablation over number of spins $N$.}
    Left: bar chart of Best (Max of MA) and Mean (of MA) as $N$ varies.
    Right: convergence curves (10-episode MA) showing learning stability.
    $N{=}3$ achieves the highest best and mean rewards, with robust, fast convergence; larger $N$ degrades both metrics.}
    \label{fig:ablation_N}
\end{figure*}

\begin{figure*}[!t]
    \centering
    \includegraphics[width=\textwidth]{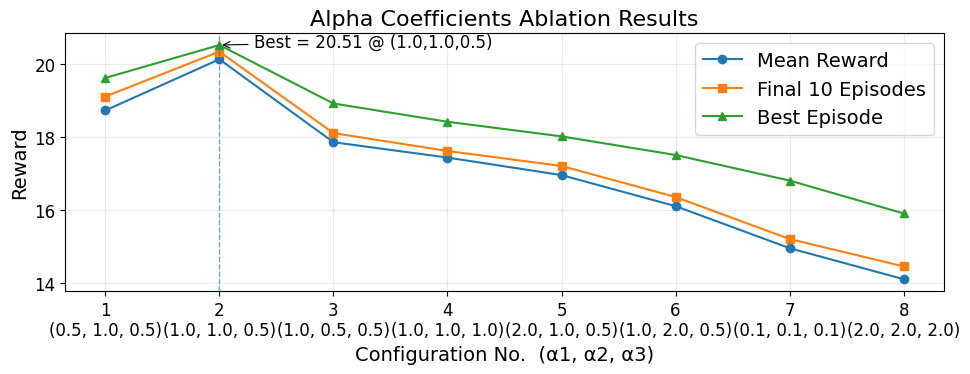}
    \caption{\textbf{Sensitivity to reward coefficients $(\alpha_1,\alpha_2,\alpha_3)$.}
    Mean reward, final-10-episode average, and best episode across eight configurations.
    The best configuration is $(1.0,\,1.0,\,0.5)$, balancing fidelity and security with a moderate emissions penalty; increasing $\alpha_3$ reduces reward as expected.}
    \label{fig:alpha_sweep}
\end{figure*}

\begin{figure}[!t]
    \centering
    \includegraphics[width=\columnwidth]{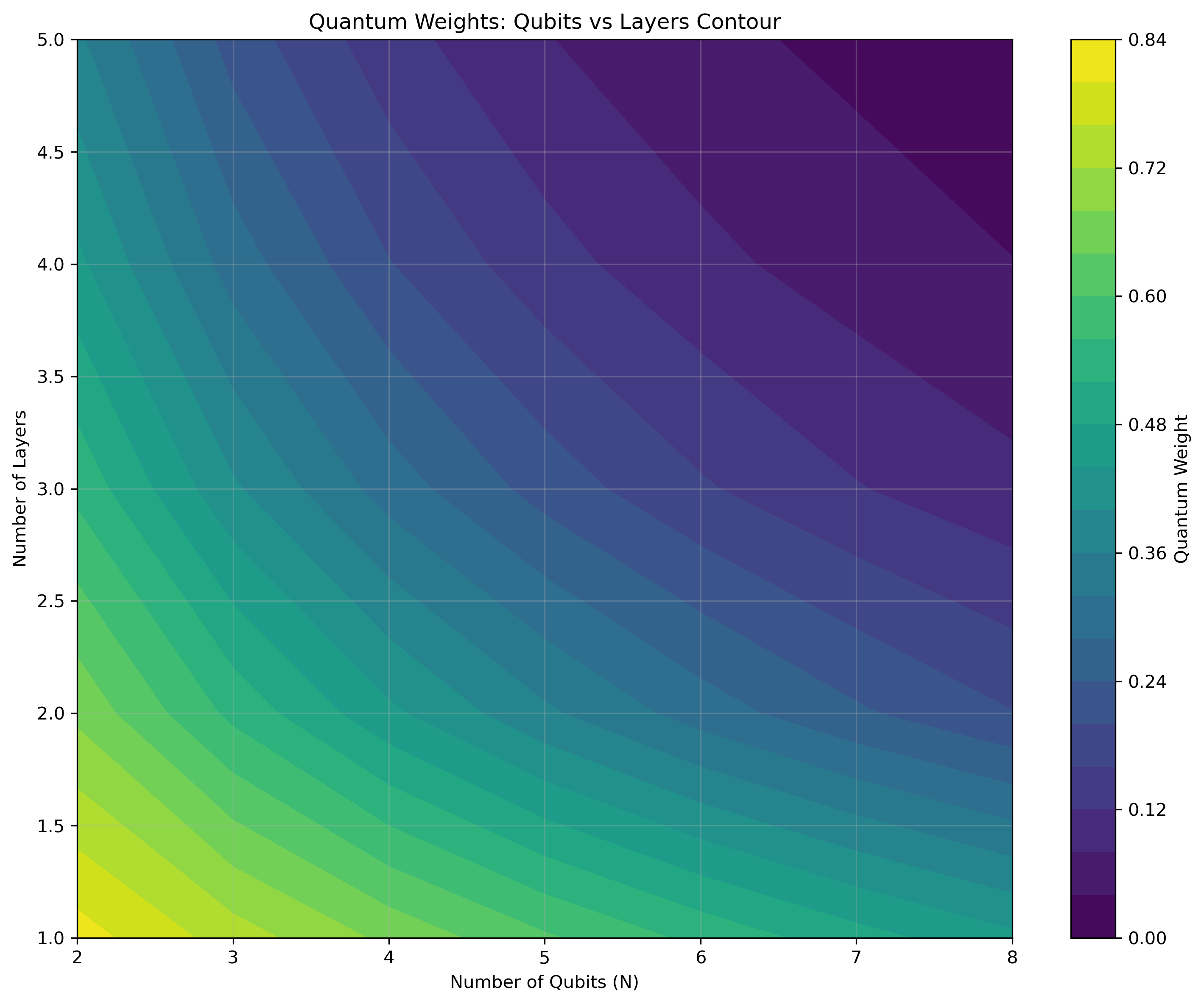}
    \caption{\textbf{Quantum weight landscape versus number of spins ($N$) and circuit depth (layers).}
    Heavier $N$ and deeper layers lower the effective weight (cooler colors), indicating stronger regularization/penalty; this trend is consistent with the empirical optimum near $N{=}3$.}
    \label{fig:quantum_weights_contour}
\end{figure}

\begin{table}[!t]
    \centering
    \caption{\textbf{Summary across learning rates (0.005 to 0.0001).}
    “Best (Max)” is the peak reward across the sweep.
    Averages are read from the curves at LR${}=0.005$ and LR${}=0.0001$.
    Canonical references are cited for each baseline.}
    \label{tab:lr_summary}
    \renewcommand{\arraystretch}{1.1}
    \resizebox{\columnwidth}{!}{%
    \begin{tabular}{lccc}
        \toprule
        \textbf{Method} & \textbf{Best (Max)} & \textbf{Avg @ 0.005} & \textbf{Avg @ 0.0001} \\
        \midrule
        Proposed (ours)                    & \textbf{20.51} & 14.85 & \textbf{20.51} \\
        PPO~\cite{schulman2017ppo}         & 20.20 & 13.95 & 19.61 \\
        DDQN~\cite{vanhasselt2016double}    & 19.60 & 13.25 & 18.91 \\
        ACKTR~\cite{wu2017acktr}            & 19.10 & 12.75 & 18.41 \\
        DDPG~\cite{lillicrap2016ddpg}       & 18.70 & 12.35 & 18.01 \\
        Ensemble RL (Bootstrapped DQN)~\cite{osband2016bootstrapped} & 18.20 & 11.85 & 17.51 \\
        ACER~\cite{wang2017acer}            & 17.80 & 11.25 & 16.91 \\
        \midrule
        GRAPE ~\cite{khaneja2005grape} & 13.00 & 13.00 & 13.00 \\
        MPC ~\cite{rawlings2009mpc}     & 12.20 & 12.20 & 12.20 \\
        Human (reference)                         & 13.10 & 13.10 & 13.10 \\
        \bottomrule
    \end{tabular}%
    }
\end{table}

\begin{table}[!t]
\centering
\caption{\textbf{Performance across $(\alpha_1,\alpha_2,\alpha_3)$ configurations.}
Bold row is the best overall: No.\,2 with $(1.0,1.0,0.5)$, Best Episode $=20.51$.}
\label{tab:alpha_results}
\begin{tabular}{lcccccc}
\toprule
No & $\alpha_1$ & $\alpha_2$ & $\alpha_3$ & Mean & Final 10 & Best \\
\midrule
1 & 0.5 & 1.0 & 0.5 & 18.72 & 19.10 & 19.61 \\
\textbf{2} & \textbf{1.0} & \textbf{1.0} & \textbf{0.5} & \textbf{20.12} & \textbf{20.33} & \textbf{20.51} \\
3 & 1.0 & 0.5 & 0.5 & 17.85 & 18.10 & 18.91 \\
4 & 1.0 & 1.0 & 1.0 & 17.43 & 17.61 & 18.41 \\
5 & 2.0 & 1.0 & 0.5 & 16.95 & 17.20 & 18.01 \\
6 & 1.0 & 2.0 & 0.5 & 16.10 & 16.35 & 17.50 \\
7 & 0.1 & 0.1 & 0.1 & 14.95 & 15.20 & 16.80 \\
8 & 2.0 & 2.0 & 2.0 & 14.10 & 14.45 & 15.90 \\
\bottomrule
\end{tabular}
\end{table}

\begin{table}[!t]
\centering
\caption{\textbf{Ablation over $N$.} Statistics computed from 10-episode moving averages (edge-padded). $N{=}3$ dominates both mean and best.}
\label{tab:ablation_N}
\begin{tabular}{lccc}
\toprule
$N$ & Mean (MA) & Final 10 (MA) & Best (Max MA) \\
\midrule
2 & 18.53 & 19.20 & 19.61 \\
\textbf{3} & \textbf{19.49} & \textbf{20.48} & \textbf{20.51} \\
4 & 17.74 & 18.25 & 18.91 \\
5 & 17.51 & 17.85 & 18.41 \\
6 & 16.83 & 17.30 & 18.01 \\
\bottomrule
\end{tabular}
\end{table}

\paragraph*{Quantum weight landscape.}
\hl{Figure}~\ref{fig:quantum_weights_contour} \hl{depicts the aggregate, normalized control-weight magnitude across model sizes} ($N$) \hl{and circuit depths (layers). Cooler colors correspond to lower effective weights, indicating a stronger regularization/penalty and a harder control regime for the policy. The landscape shows a clear trend: larger} $N$ \hl{and deeper circuits move the system into cooler regions, which correlates with the observed drop in mean and best rewards in our} $N$-\hl{ablation (optimal near} $N{=}3$). \hl{This provides a structural explanation for the empirical sweet spot: modest system size and depth balance expressivity with controllability and training stability.}

\begin{figure}[!t]
    \centering
    \includegraphics[width=\columnwidth]{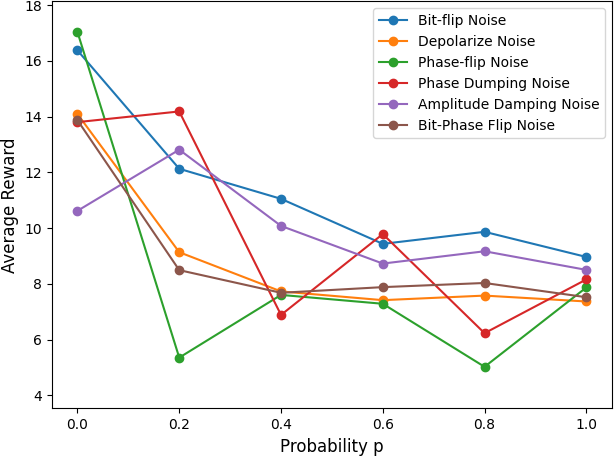}
    \caption{\textbf{Impact of quantum noise.}
    Average rewards under increasing noise probabilities for several channels (e.g., bit-flip, depolarizing, phase-flip). The proposed method degrades gracefully.}
    \label{fig:noise_impact}
\end{figure}

\subsection{Impact of Quantum Noise} \label{subsec:QuantumNoiseImpact} \hl{In order to assess the resilience of the learning framework in the presence of quantum imperfections, we introduced various noise channels (e.g., bit-flip, depolarizing, and phase-flip noise) at different probabilities.} Figure~\ref{fig:noise_impact} \hl{presents a comprehensive view of the average reward as the noise probability \textit{p} increases. As anticipated, higher noise levels generally degrade performance; however, the proposed agent manages to maintain relatively higher rewards despite noise-related disruptions. This suggests that the quantum-enhanced RL formulation, coupled with a robust policy update mechanism, is more tolerant to quantum hardware errors. These findings indicate that incorporating quantum elements into RL, alongside classical techniques such as experience replay and target network updates, can alleviate the detrimental effects of noise to a certain extent.}

\section{Conclusion, Limitations, and Future Work}
\label{sec:conclusion}
\hl{In conclusion, we presented a quantum-inspired reinforcement learning framework that unifies inventory management, carbon-emissions penalization, and security objectives within a single decision model. Leveraging a controllable spin-chain analogy coupled to AIoT signals, the method operationalizes a multi-objective reward and learns robust policies via value-based and ensemble policy updates. Across learning-rate sweeps and ablations, the proposed controller achieved stable convergence and superior peak/late-episode returns relative to strong RL and model-based baselines, with an empirically favorable operating point.}

\hl{Despite these strengths, several limitations warrant attention. Certain quantum-inspired components rely on idealized assumptions, and our noise analysis is simulation-based; we have not yet implemented or measured these processes on physical AIoT or quantum-analog hardware. In practice, we plan to (i) perform hardware-in-the-loop experiments with realistic sensor/actuator paths, (ii) calibrate channel parameters (bit-/phase-flip, depolarizing, erasure) to device-level error statistics capturing dropout, jitter, and bursty congestion, and (iii) extend to correlated and non-Markovian noise. Additionally, the multi-objective reward can grow complex at scale for interdependent networks. Future work will investigate more rigorous quantum formulations, distributed or hierarchical RL, and closed-loop feedback from decentralized IoT devices. These steps will strengthen the mapping from abstract channels to measured disturbances and further enable secure, sustainable, and efficient next-generation supply chains.}

\section*{Acknowledgments}
This work has been partially funded by the European Union’s Horizon Europe Framework Programme under grant agreement No.\ 101225776 (project \textit{SecQDevOps}).

\bibliographystyle{IEEEtran}
\bibliography{references}

\vfill

\end{document}